\documentclass[letterpaper, 10 pt, conference]{ieeeconf}

\IEEEoverridecommandlockouts                                 
\usepackage{graphicx} 

\usepackage{amsmath} 

\usepackage{hyperref}
\usepackage{amssymb}
\usepackage[utf8]{inputenc}
\usepackage[T1]{fontenc}
\usepackage{mathtools}
\usepackage[thinc]{esdiff}
\usepackage{amsmath}
\usepackage{algorithm}
\usepackage{algpseudocode}

\usepackage{cite}
\usepackage{amsfonts}
\usepackage{textcomp}

\usepackage{booktabs}

\newcounter{desccount}
\newcommand{\descitem}[1]{
  \item[#1] \refstepcounter{desccount}\label{#1}}
\newcommand{\descref}[1]{\hyperref[#1]{#1}}
\usepackage{booktabs}

\usepackage{comment}
\usepackage{multirow}
\usepackage{color}
\usepackage{url}
\usepackage{subcaption}
\newtheorem{example}{Example}

\newtheorem{theorem}{Theorem}[section]

\newtheorem{lemma}[theorem]{Lemma}
\newtheorem{proposition}[theorem]{Proposition}
\newcommand{\free}{\mathcal{M}_{\phi}} 
\newcommand{\lossfree}{\mathcal{L}_{\phi}}

\newcommand{\nin}{c}

\newcommand{\norm}[1]{\left\lVert#1\right\rVert}

\newcommand{\sys}{S}

\newcommand{\D}{D}
\newcommand{\cumul}{F_{\ell}(l)}
\newcommand{\VARa}{V\!aR_\alpha [\lossfree]}

\title{Distributionally robust minimization in  meta-learning  for system identification }
\author{Matteo Rufolo , Dario Piga, Marco Forgione
\thanks{The authors are with the IDSIA Dalle Molle Institute for Artificial Intelligence USI-SUPSI, Via la Santa 1, CH-6962 Lugano-Viganello, Switzerland. 
 e-mails: {matteo.rufolo, dario.piga, marco.forgione@supsi.ch}}  }
\date{January 2025}
\begin{document}

\maketitle

\begin{abstract}
Meta learning aims at learning how to solve tasks, and thus it allows to estimate models that can be quickly adapted to new scenarios. This work explores distributionally robust minimization in meta learning for system identification. Standard meta learning approaches optimize the expected loss, overlooking task variability. We use an alternative approach, adopting a distributionally robust optimization paradigm that prioritizes high-loss tasks, enhancing performance in worst-case scenarios. Evaluated on a meta model trained on a class of synthetic dynamical systems and tested in both in-distribution and out-of-distribution settings, the proposed approach allows to reduce failures in safety-critical applications.
\end{abstract}

\section{Introduction}

Classic system identification (SYSID) models each system independently, using only its own input-output data and ignoring knowledge from similar systems. This often leads to suboptimal results when data is scarce. A promising alternative is to integrate \emph{meta learning}, a framework introduced in the 1980s~\cite{schmidhuber1987evolutionary} and recently revitalized for its ability to enable rapid adaptation across related tasks. By training a \emph{meta-learner} on a distribution of similar systems, the model can generalize efficiently to  unseen dynamics with minimal additional data~\cite{hospedales2021meta}. This approach offers a paradigm shift in SYSID, enhancing data efficiency and generalization.

Within the SYSID framework, a meta-learning approach treats the identification of different dynamical systems as a series of related tasks, with the aim of generalizing across multiple systems and efficiently adapt to newly encountered ones~\cite{chakrabarty2023meta}. 
One particular meta-learning paradigm used in SYSID is in-context learning, see~\cite{forgione2023context}. In this approach, a large model is pre-trained on a broad set of related datasets, allowing it to process a \textit{context} sequence consisting of input-output data from a system, along with a \textit{query} input sequence, and subsequently predict the corresponding query output without requiring task-specific re-training. Pre-training  is performed offline in a supervised setting over \textit{synthetic datasets}, derived from random input signals applied to a diverse set of randomly sampled systems within a given class. In contrast, inference follows a zero-shot learning paradigm, wherein the pre-trained model directly provides predictions for a new system  without requiring additional fine-tuning. {Unlike traditional SYSID methods, this approach trains a meta model to generalize across related systems, eliminating the need for system-specific identification and prior knowledge of the underlying dynamics.}
The meta model functions as a \emph{learned} SYSID algorithm that identifies system-specific models from the context and utilizes this knowledge to generate output predictions for the query input. 

An important characteristic of existing meta-learning approaches is their focus on minimizing the expected value of the loss over a collection of tasks. This approach implicitly assumes that all tasks are equally important, disregarding differences in task complexity or loss magnitude. While effective in many cases, this uniform weighting can be problematic in safety-critical applications.
\begin{example}
{Consider an autonomous vehicle trained to follow urban trajectories using data from diverse weather conditions. Standard training assigns equal weight to all scenarios, following the expected risk minimization principle with a uniform task weighting strategy.  While effective for common conditions, it may lead to poor performance in rare but critical cases, such as snowstorms, where obscured lane markings and reduced traction pose significant challenges.}
\end{example}

To address robustness issues in meta-learning, recent studies have proposed alternative strategies. Some leverage adversarial learning to improve generalization under distribution shifts~\cite{gold2020advmeta}, while others reshape the training distribution by emphasizing high-loss tasks, thus focusing learning on the most \textit{problematic} instances~\cite{Collins2020TRMAML,wang2023simpleyet}.

In this paper, we adapt the {distributionally} robust meta-learning formulation introduced in~\cite{wang2023simpleyet}, based on the optimization of a tail loss~\cite{rockafeller2000opt}, to the in-context system identification framework. {This formulation aligns well with meta-learning, where variability and shifts across tasks demand distributional robustness for reliable generalization. This adaptation relies on the premise that tasks vary in difficulty, with more complex ones offering greater training informative content. Our goal is to exploit these tasks more effectively, following a strategy inspired by active learning~\cite{bemporad2023active}.}
 Specifically, after mathematically formulating the robust meta-learning problem for SYSID, we employ a heuristic, adapted from~\cite{rahimian2022distrobust}, to solve the distributionally robust optimization problem. {We then analyze its convergence properties, leveraging theoretical insights from the original formulation}{, but applying them to our in-context framework and providing a more practical interpretation}. With respect to~\cite{wang2023simpleyet}, we refine the convergence assumptions to  suit the in-context SYSID setting considered in this paper, providing a practical interpretation of these assumptions in the given context.

The rest of the paper is structured as follows. Section~\ref{sec:pb_desc} introduces the problem setting, presenting the {probabilistic} meta modeling framework from~\cite{rufolo2024enhanced} and discussing the limitations of the conventional expected loss formulation. Section~\ref{sec:math_form} focuses on distributionally robust minimization for meta-learning in SYSID. This includes a heuristic for robust optimization, along with new theoretical results analyzing the convergence properties of the proposed approach. Section~\ref{sec:example} reports numerical examples demonstrating improved robustness in both in-distribution {(ID)} and out-of-distribution {(OOD)} scenarios, comparing also the performance of the meta approach with classical SYSID approaches. To provide a
more comprehensive assessment, an additional analysis is
included at the end of Section~\ref{sec:example}, investigating the role
of probabilistic estimation within the distributionally robust
framework. Specifically, the proposed method is compared
to a non-probabilistic variant trained using the Root Mean
Squared Error (RMSE) loss, following~\cite{forgione2023context}, which does not
incorporate uncertainty information.

\section{In-context system identification}
\label{sec:pb_desc}
{
In this paper, following the footsteps of~\cite{forgione2023context, rufolo2024enhanced}, an in-context learner $\free$ referred to as a \emph{meta model} is employed for probabilistic modeling of \emph{a class} of dynamical systems, taking advantage of the in-context capabilities of Transformer architectures~\cite{achiam2023gpt}.
}

The meta model $\free$ {is trained to provide probabilistic output estimates over a prediction horizon}, on a potentially \emph{infinite stream} of datasets 
\(
\{ \D^{(i)} = (u_{1:n}^{(i)}, y_{1:n}^{(i)}), \, i = 1, 2, \dots, \infty \}\), {realizations of a random variable with distribution \( p(\mathcal{D}) \)}.
Each dataset $\D^{(i)}$ is obtained by applying  a random input sequence $u_{1:n}^{(i)}$, with $u_k^{(i)}\in \mathbb{R}^{n_u}$, to a  dynamical system $\sys^{(i)}$ randomly sampled within the class, resulting in the corresponding output sequence $y_{1:n}^{(i)}$, with $y_k^{(i)} \in \mathbb{R}^{n_y}$.\footnote{In the following, for notational convenience, we consider the case $n_y=1$. The extension to the multi-output case $n_y>1$ is trivial.} Formally, the datasets $\D^{(i)}$ can be seen as {independent} realizations of a random variable with distribution $p(\D)$, induced by the random generation of systems and inputs, whose analytical form is unknown. 
{Since all datasets are drawn from the same distribution $p(D)$, assumed to be well-behaved\footnote{
{The specific conditions that $p(\mathcal{D})$ must satisfy are stated in Proposition~3.2.}
}, it is intuitive that their shared structure enables few-shot adaptation, allowing knowledge to be transferred effectively across different realizations.}

Following the notation of~\cite{rufolo2024enhanced}, which extends~\cite{forgione2023context} and frames in-context SYSID in a probabilistic setting, each system $\sys^{(i)}$ generates two input-output sequences. The first sequence, referred to as the \emph{context} $(u^{(i)}_{1:m}, y^{(i)}_{1:m})$, implicitly informs the meta model about the system's dynamics. The second sequence, denoted as the \emph{query} $(\tilde{u}^{(i)}_{1:n}, \tilde{y}^{(i)}_{1:n})$, represents the data to be predicted.
We refer to \emph{task} as a combination of the two previously defined sequences, namely:
\(
\{ \D^{(i)} = ((u_{1:m}^{(i)}, y_{1:m}^{(i)}),(\tilde{u}_{1:n}^{(i)}, \tilde{y}_{1:n}^{(i)})), \, i = 1, 2, \dots, \infty \}
\). 

{Throughout the paper, we use the encoder-decoder Transformer architecture proposed in~\cite{rufolo2024enhanced} as meta model $\free$.}
The meta model processes three components: the input-output context sequence $(u_{1:m}, y_{1:m})$, the initial portion of the input-output query sequence $(\tilde{u}_{1:\nin}, \tilde{y}_{1:\nin})$, used as initial conditions, and the remaining input sequence $\tilde{u}_{\nin+1:n}$.
It then outputs the estimated mean $\mu_{\nin+1:n}$ and standard deviation $\sigma_{\nin+1:n}$ of the predicted query output distribution: 
\begin{equation}
\label{eq:Nodel_free_sim_CI}
\mu_{\nin+1:n}, \sigma_{\nin+1:n} =
\free(\tilde{u}_{\nin+1:n}, \tilde{u}_{1:{\nin}}, \tilde{y}_{1:{\nin}}, u_{1:m}, y_{1:m}).
\end{equation}
To compact the notation, we define the set of inputs processed by the meta model as $X = (\tilde{u}_{\nin+1:n}, \tilde{u}_{1:\nin}, \tilde{y}_{1:\nin}, u_{1:m}, y_{1:m})$ and the corresponding outputs as $\theta = (\mu_{\nin+1:n}, \sigma_{\nin+1:n})$. We have then: $\theta = \free(X)$.

{Both the standard meta learning loss used in~\cite{rufolo2024enhanced} and the distributionally robust one introduced in this paper are defined in terms of the}  Kullback-Leibler (KL) divergence between the unknown query output density $p(\tilde{y}_{\nin+1:n} | X)$ and the predicted query output distribution. The latter is parameterized as a Gaussian 
 $q_{\theta(\phi,X)}(\tilde{y}_{\nin+1:n})$, with  mean and standard deviation generated by $\free$. Thus, the risk function is defined as:
\begin{multline}
    \label{eq:risk_function}
\ell(\tilde{y}_{\nin+1:n}, X;\phi)  =   \log \frac{p(\tilde{y}_{\nin+1:n} | X)} {q_{\theta(\phi,X)}(\tilde{y}_{\nin+1:n})} \\
        = -\log q_{\theta(\phi,X)}(\tilde{y}_{\nin+1:n})+K\\
        =\frac{1}{2} \log \left( 2\pi \sigma_{\nin+1:n} \right) + \frac{1}{2} \left( \frac{\tilde{y}_{\nin+1:n} - \mu_{\nin+1:n}}{\sigma_{\nin+1:n}} \right)^2+K,
\end{multline}
where K is a constant independent of $\phi$. 

{
\subsection{Standard learning: expected KL loss minimization}
}
\label{subsec:learning_procedure}
In line with standard meta-learning practice,  
{$\free$ was trained in \cite{rufolo2024enhanced}} by minimizing the \emph{expected value} $J(\phi)$ of the risk function~\eqref{eq:risk_function} over the dataset distribution $p(\D)$, i.e.,  
\begin{equation}
    \label{eq:vanilla_loss}
    \phi^* = \arg \min_{\phi} 
    \underbrace{
        \mathbb{E}_{p(\D)}
        \left[
        -\log q_{\theta(\phi,X)}(\tilde{y}_{\nin+1:n})
        \right]
    }_{= J(\phi)},
\end{equation}
{which is the expected KL divergence between $p(\tilde{y}_{\nin+1:n} | X)$ and  $q_{\theta(\phi,X)}(\tilde{y}_{\nin+1:n})$ .}
In practice, since the functional form of $p(\D)$ is unknown and, in any case, evaluating the expectation above is generally intractable, the Monte Carlo average:
\begin{equation}
   \label{eq:J_prob}
     \tilde{J}(\phi) =  \frac{1}{b} \sum_{i=1}^b -\log q_{\theta(\phi,X^{(i)})}(\tilde{y}^{(i)}_{\nin+1:n})
\end{equation}
computed over $b$ datasets sampled from $p(\D)$ {was} used to approximate $J(\phi)$.
The loss $\tilde{J}(\phi)$ {was} then minimized using stochastic minibatch gradient descent, where $b$ datasets are resampled from $p(\D)$ at each iteration, with $b$ serving as the {minibatch size} in the optimization process.
However, the uniform task weighting corresponding to the expected KL risk minimization \eqref{eq:vanilla_loss} can be problematic in applications where users not only aim to optimize the average performance but also wish to mitigate  failures in the worst-case scenarios~\cite{rockafeller2000opt}. 

\section{Robust learning: tail KL loss minimization} 

\label{sec:math_form}
{In order to avoid these failures, in this work we analyze the meta-learning framework from the perspective of the task distribution $p(\D)$, emphasizing the importance of tasks with worse performance. To this aim, we modify the optimization problem~\eqref{eq:vanilla_loss} by minimizing the expected \emph{tail loss} of the task distribution, defined as the expectation over tasks with higher risk values. This approach aims to mitigate performance disparities across tasks and enhance robustness against the most challenging instances.} In order to  minimize the expected tail loss over the task distribution,  we focus on dataset realizations whose associated risk function~\eqref{eq:risk_function} exceeds a predefined threshold. To formally define this objective, we begin by introducing the necessary mathematical framework.

We denote the support of $p(\D)$ by $\mathbb{D}$. From \eqref{eq:risk_function}, for a given  parameter vector $\phi$, we define the meta-learning operator $\lossfree$ that maps each dataset realization to its respective risk:
\begin{equation}
  \lossfree: \mathbb{D} \rightarrow   \mathbb{R}; \;\;\; \D^{(i)} \mapsto \ell(\tilde{y}_{\nin+1:n}^{(i)}, X^{(i)};\phi).
  \label{eq:meta_operator}
\end{equation}
Note that $\lossfree$ is a deterministic operator.
Since the sequences $\tilde{y}_{\nin+1:n}^{(i)}$ and $ X^{(i)}$ uniquely define the dataset realization $\D^{(i)}$, in the rest of the paper the dependence of $\ell$ on $\tilde{y}_{\nin+1:n}^{(i)}$ and $X^{(i)}$ will be written in the compact form:
    $\ell(\tilde{y}_{\nin+1:n}^{(i)}, X^{(i)};\phi) = \ell(\D^{(i)};\phi).$

Using the operator~\eqref{eq:meta_operator}, the risk function can be viewed as a random variable with the corresponding \emph{probability density function} (PDF) \( p_{\ell}(l) \) and cumulative distribution function (CDF) \( \cumul = P\{\ell(\D;\phi) \leq l, \D \in \mathbb{D}, l \in \mathbb{R}\} \). For simplicity, we abbreviate \( \ell(\D;\phi) \) as \( \ell \).

From the definition of the CDF \( \cumul \), we define a quantile known as the Value-at-Risk (VaR)~\cite{rockafeller2000opt} for the random variable \( \ell \), at any probability level \( \alpha \in (0, 1] \) (confidence interval), given the density \( p(\D) \) and the meta model parameters \( \phi \):
\begin{equation}
\label{eq:VaR}
    \VARa = \inf_{l\in\mathbb{R}}\{l\mid \cumul \geq \alpha\}.
\end{equation}

Let us focus on the domain of the random variable \( \ell \) with \( \ell(\D;\phi) \geq \VARa \), namely:
    $\mathbb{D}_\alpha := \{\D \in \mathbb{D} \;|\; \ell(\D;\phi) \geq \VARa \}.$

We can then define the conditional expectation over $\mathbb{D}_\alpha$, known as the Conditional Value-at-Risk (CVaR)~\cite{rockafeller2000opt}, as:
\begin{equation}
\label{eq:CVaR}
    C\VARa = (1-\alpha)^{-1}\int_{\mathbb{D}_{\alpha}}\ \ell(\D;\phi) \, p(\D) \, d\D.
\end{equation}
The primary objective of the robust learning problem addressed in this work is to minimize  \( C\VARa \) in \eqref{eq:CVaR}, which focuses   on the tail of the risk function distribution, i.e., dataset realizations \( \D^{(i)} \) where \( \ell(\D^{(i)};\phi) \) exceeds \( \VARa \). To this aim, we introduce a new probability distribution \( p_\alpha(\D;\phi) \) over the restricted dataset support $\mathbb{D}_\alpha$: 
\begin{equation}
    p_\alpha(\D;\phi)  := \frac{p(\D)}{\mathcal{Z}_\alpha(\phi)},\text{ with } \D\in\mathbb{D}_\alpha,
    \label{eq:updated_prob}
\end{equation}
where   $\mathcal{Z}_\alpha(\phi)$ is the normalization factor over the updated support, i.e. $\mathcal{Z}_\alpha(\phi)=\int_{\mathbb{D}_\alpha} p(\D)d\D$. The probability distribution $ p_\alpha(\D;\phi)$ in \eqref{eq:updated_prob} thus characterizes dataset realizations whose risk function $\ell$ is larger then the $\VARa$ in \eqref{eq:VaR}.

By substituting the definition of the risk function~\eqref{eq:risk_function} into~\eqref{eq:CVaR} and from the definition of $p_\alpha(\D;\phi)$ in \eqref{eq:updated_prob}, minimization of the $C\VARa$ is equivalent to solving the  optimization problem: 
\begin{equation}
    \label{eq:robust_loss}
    \phi^* = \arg \min_{\phi} 
    \underbrace{
        \mathbb{E}_{p_\alpha(\D; \phi)}
        \left[
        -\log q_{\theta(\phi,X)}(\tilde{y}_{\nin+1:n})
        \right]
    }_{= J_\alpha(\phi)}.
\end{equation}
However, since neither \( p(\D) \) nor $p_\alpha(\D; \phi)$ is known in closed form, the expectation over $p_\alpha(\D; \phi)$ in~\eqref{eq:robust_loss} cannot be computed analytically. Moreover, as there is no straightforward way to sample from $p_\alpha(\D; \phi)$, directly approximating the expectation using a Monte Carlo average is not feasible. To address this challenge, inspired by~\cite{wang2023simpleyet}, we adopt the quantile estimation strategy described in Algorithm~\ref{alg:robust_training}. The main idea of the quantile estimation strategy is to sample from $p(\D)$ a batch of $b$ tasks (step 2), and then to approximate  the loss $J_{\alpha}$ by only considering the $\lfloor (1-\alpha) b \rfloor$ tasks with the highest risk (steps 4-5), where $\lfloor \cdot \rfloor$ denotes the floor operator. { This subset of tasks, denoted by \( \mathcal{B}_{\alpha} \), is used to define the loss approximation \( \tilde{J}_{\alpha} \). The gradient of \( \tilde{J}_{\alpha} \) is then employed to update the meta model parameters \( \phi \) via standard gradient descent (steps 6-7).}
\begin{algorithm}[ht] 
    \caption{Gradient-based robust Optimization with Quantile Estimation}
    \label{alg:robust_training}
    \begin{algorithmic}[1]
        \Require Initial meta model parameters \( \phi^{(0)} \), batch size \( b \), confidence interval \( \alpha \in (0,1] \), maximum iterations \( \text{max\_iter} \), learning rate $\eta>0$
        \Ensure Optimized meta model parameters \( \hat{\phi} \)
        
        \For{\( j = 1 \) to \( \text{max\_iter} \)}
            \State Sample a batch \( \mathcal{B} = \{\D^{(i)}\}_{i=1}^{b},\; D^{(i)} \sim p(\D) \);
            \State Compute the risk function \( l_i = \ell(\D^{(i)}; \phi^{(j)}) \) for each \( \D^{(i)} \in \mathcal{B} \);
            \State Select the top \( \lfloor (1-\alpha) b \rfloor \) 
             datasets with the highest risk, forming the subset \( \mathcal{B}_{\alpha} \);
            \State Approximate the loss $J_{\alpha}$ using the subset $\mathcal{B}_{\alpha}$ as $\tilde{J}_\alpha$;
            \State Compute the gradient \( \nabla_{\phi} \tilde{J}_\alpha \) over the batch \( \mathcal{B}_{\alpha}\);
            \State Update parameters: \( \phi^{(j+1)} \gets \phi^{(j)} - \eta \nabla_{\phi} \tilde{J}_\alpha \).
        \EndFor
        
        \State \Return Optimized parameters \( \hat{\phi} \)
    \end{algorithmic}
\end{algorithm}
Properties of the quantile estimation strategy in Algorithm~\ref{alg:robust_training} are reported in the following Lemma from \cite{wang2023simpleyet}.
Before stating the lemma, we define the Monte Carlo approximation of the $\VARa$:
\begin{equation}
    \tilde{V}[\mathcal{L}_{\phi}] := \min_{i = 1,\ldots,\lfloor(1-\alpha) b \rfloor} \{\ell(\D^{(i)};\phi)|\D^{(i)}\in\mathcal{B}_{\alpha}\}.
    \end{equation}
\begin{lemma} 
    Suppose the following assumptions hold:
\begin{description}
    \item[(I)]\label{first} the risk function \( \ell(\D; \phi) \) is Lipschitz continuous \emph{w.r.t.}  \( \phi \),
    \item[(II)]\label{second}  the CDF \( \cumul \) is Lipschitz continuous \emph{w.r.t.} \( l \) with Lipschitz constant $\beta_\ell$,
    \item[(III)]\label{third} the Monte Carlo $\VARa$ approximation error  is bounded, i.e. $|\tilde{V}[\mathcal{L}_{\phi}]-\VARa| \leq \frac{\eta}{\beta_\ell(1-\alpha)^2}$ at every iteration of Algorithm \ref{alg:robust_training} ,
    \item[(IV)]\label{fourth} the PDF \( p_{\alpha}(\D; \phi) \) is Lipschitz continuous \emph{w.r.t.} \( \phi \),
    \item[(V)]\label{fifth} the risk function \( \ell(\D; \phi) \) over the dataset realizations drawn from \( p_{\alpha}(\D; \phi) \) is bounded for any \( \phi \), s.t. $\norm{\phi}\leq C<+\infty$\footnotemark[3], with $C\in \mathbb{R}$.
\end{description}
\footnotetext[3]{The boundedness restriction on the parameter $\phi$ was not explicitly stated in~\cite{wang2023simpleyet}, but it is actually needed for the Assumption (V).}
{Then Algorithm~\ref{alg:robust_training} converges to a local optimum of the robust minimization problem~\eqref{eq:robust_loss}, with the improvement guarantee
      $J_\alpha(\phi^{(j+1)})\leq J_\alpha(\phi^{(j)})$

      }
\label{lemma:local_opt_paper}
\hfill $\square$ 

\end{lemma}

{Among the assumptions in Lemma 3.1, the strongest are (V) and (III). Assumption (V) is relatively mild, especially given the Lipschitz condition (I). In contrast, Assumption (III) is the most restrictive. Unlike standard Monte Carlo estimation, which typically requires only bounded variance, this condition enforces a direct bound on the estimation error at each iteration. While stronger, it remains reasonable in our context. In particular, the bound becomes looser (i.e., more permissive) for larger $\alpha$, where the tail quantiles are harder to approximate. For small $\alpha$, the bound is tighter, but the approximation is easier due to the statistical efficiency of Monte Carlo estimation in this regime.}
In the following we adapt the assumptions of Lemma~\ref{lemma:local_opt_paper} to make them suited for the Transformer architecture used in~\cite{rufolo2024enhanced} (and also adopted in this work) to describe the meta model $\free$.
\begin{proposition}
{Suppose that, in addition to Assumptions (I), (III), and (IV) of Lemma~\ref{lemma:local_opt_paper}, the following hold:
\begin{description}
    \descitem{(1)} \label{item3} there exists \( M \in \mathbb{R} \) such that \( \sup_{\D \in \mathbb{D}} p(\D) \leq M \),
    \descitem{(2)} \label{item4}the measure of \( \{ \D \in \mathbb{D} \mid \ell(\D; \phi) = L \} \) is bounded for all \( L \in \mathbb{R} \)
    \descitem{(3)} \label{item5}the standard deviation in the parametrized distribution \( q_{\theta(\phi,X)} \) in \eqref{eq:robust_loss} is nonzero,
    \descitem{(4)} \label{item6}the activation functions used in the Transformer describing the meta model are Lipschitz continuous with respect to the processed input for any \( \phi \), s.t.  $\norm{\phi}\leq C<+\infty$., with $C\in \mathbb{R}$,
    \descitem{(5)} \label{item7}every input processed by Layer normalization  in the Transformer has nonzero standard deviation for any \( \phi \) s.t. $\norm{\phi}\leq C<+\infty$, with $C\in \mathbb{R}$.
\end{description}

 Then the statement of Lemma~\ref{lemma:local_opt_paper} holds. 
  }
\label{prop:local_opt}
\hfill $\square$
\end{proposition}

\begin{proof}[Proof of \textit{Proposition~\ref{prop:local_opt}}]

{Since Assumptions (I), (III), and (IV) are already shared with Lemma~\ref{lemma:local_opt_paper}, it only remains to show that Assumptions (1)–(5) imply Assumptions (II) and (V) of the Lemma.}

We now prove that Assumption (II) in Lemma~\ref{lemma:local_opt_paper} holds under the assumptions of Proposition~\ref{prop:local_opt}. Given \( L_1, L_2 \in \mathbb{R}^+ \), assuming without loss of generality $L_2 \geq L_1$, we have:
\begin{equation*}
    |F_{\ell}(L_2; \phi) - F_{\ell}(L_1; \phi)| \leq \sup_{l \in \mathbb{R}} p_{\ell}(l) (L_2 - L_1).
\end{equation*}
We need to prove that $\sup_{L \in \mathbb{R}} p_{\ell}(L)$ is bounded.  
Since the map $\lossfree$ in \eqref{eq:meta_operator} is deterministic, the only source of randomness in $\ell$ is the random variable $\D$. Thus, we obtain:
\begin{align}
p_{\ell}(L) &= \int_{ \mathbb{D}} p(\D) p(L|\D) \, d\D 
= \int_{\mathbb{D}} p(\D) \delta(L - \ell(\D, \phi)) \, d\D \nonumber \\
&\leq M \int_{\mathbb{D}} \delta(L - \ell(\D, \phi)) \, d\D
\label{bound_pl}
\end{align}
where ${\phi}$ is fixed from the definition in \eqref{eq:meta_operator} and the last inequality follows from Assumption~\eqref{item3}. From inequality~\eqref{bound_pl} and  Assumption~\eqref{item4}, $p_{\ell}(L)$ is bounded for all $L\in\mathbb{R}$, thus \(  \sup_{L \in \mathbb{R}} p_{\ell}(L) \) is bounded and  \( \cumul \) is Lipschitz continuous \emph{w.r.t.} \( \ell \). Therefore, Assumption (II) of Lemma \ref{lemma:local_opt_paper} is satisfied.

We now prove that, under the assumptions of Proposition~\ref{prop:local_opt}, Assumption (V) is satisfied, which states that the risk function \(\ell(\D; \phi) \) evaluated over realizations of \( p_{\alpha}(\D; \phi) \) is bounded for any \( \phi \) s.t. $\norm{\phi}<+\infty$. From the definition of \(\ell(\D; \phi) \) in~\eqref{eq:risk_function}, Assumption (V) is satisfied if: 
\begin{equation}
    0 < \sigma^2_{\phi}(X) < +\infty, \quad |\tilde{y} - \mu_{\phi}(X)| < + \infty.
    \label{eq:boundness}
\end{equation}
Using Assumption~\eqref{item5} and since the true output \(\tilde{y}\) is bounded, proving that \eqref{eq:boundness} holds only requires that \(\mu_{\phi}(X)\) and \( \sigma_{\phi}(X) \) are bounded for any $X$. For this, it is sufficient to show that the Transformer architecture is Lipschitz continuous in \( X \). Indeed, if this condition holds, then, since all inputs are bounded, the corresponding output parameters remain bounded for any fixed \( \phi \) s.t. $\norm{\phi}<+\infty$. Since the inputs lie within a bounded domain, it is sufficient to verify local Lipschitz continuity \emph{w.r.t.} \( X \), which holds provided that all Transformer layers satisfy this property. In fact:
\begin{itemize}
    \item linear layers are inherently Lipschitz continuous;
    \item positional encodings are trivially Lipschitz continuous;
    \item recurrent neural networks and multilayer perceptrons (MLPs) are Lipschitz continuous, as established in \cite{qi2023lipsformer};
    \item the attention mechanism has a local Lipschitz bound proportional to the square root of its input length \cite{castin2024how};
    \item layer normalization is Lipschitz continuous under Assumption~\eqref{item7}, as shown in \cite{qi2023lipsformer}.
\end{itemize}
Using also Assumption~\eqref{item6}, the Transformer is Lipschitz continuous \emph{w.r.t.} its inputs, ensuring that the outputs \( \mu_{\phi}(X) \) and \( \sigma_{\phi}(X) \) remain bounded for any \( \phi \) s.t. $\norm{\phi}\leq C<+\infty$.

In summary, the assumptions of Proposition~\ref{prop:local_opt} imply those of Lemma~\ref{lemma:local_opt_paper}. Therefore, the conclusion of Lemma~\ref{lemma:local_opt_paper} holds under the assumptions of Proposition~\ref{prop:local_opt}.  
\end{proof}

{Proposition~\ref{prop:local_opt} aims to tailor the assumptions of Lemma~\ref{lemma:local_opt_paper} in a form more interpretable within a probabilistic in-context learning framework, specifically using a Transformer-based meta model $\free$. While its assumptions are stricter than those in Lemma~\ref{lemma:local_opt_paper}, they offer clearer practical implications. Notably, several of the revised assumptions are readily satisfied in practice. For example, Assumption~\eqref{item5} holds when the predicted standard deviation, provided as output by the meta model, is constrained to be positive via an exponential function~\cite{rufolo2024enhanced}, and Assumption~\eqref{item6} is met by common deep learning  function activations. This refinement enhances practical relevance without weakening the theoretical foundation.}

\section{Numerical examples}
\label{sec:example}
The proposed tail loss optimization strategy is used to train a meta model using synthetic data from a class of dynamical systems with a Wiener-Hammerstein (WH) structure. The meta model adopts the Transformer architecture from~\cite{rufolo2024enhanced}, comprising approximately 5.5 million parameters. Key hyperparameters include an initial condition length $\nin = 30$, query length $n - \nin = 100$, and context length $m = 400$. All experiments were repeated across different random initializations and dataset realizations, consistently yielding similar results. Training was conducted on an Nvidia RTX 3090 GPU. {To ensure reproducibility, the complete codebase is available at \href{https://github.com/mattrufolo/sysid-robust-transformer}{https://github.com/mattrufolo/sysid-robust-transformer}, which also contains an extended version of the paper with additional analyses and results.}
\subsection{System Class and Task Distribution}
The WH systems used for meta-model training follow a cascade structure \( G_1\!-\!F\!-\!G_2 \), where \( G_1 \) and \( G_2 \) are discrete-time LTI systems and \( F \) is a static nonlinearity. The LTI components have randomly generated orders in \( [1, 10] \), pole magnitudes in \( (0.5, 0.97) \), and phases in \( (0, \pi/2) \).
 The nonlinearity \( F \) is modeled {as a MLP with one hidden layer of 32 units and hyperbolic tangent activations}, with weights sampled from a normal distribution with Kaiming scaling.

Inputs consist of white noise signals drawn from standard Gaussian, and outputs are standardized to zero mean and unit variance. The output are then corrupted with measurement noise, modeled by white Gaussian noise with standard deviation $\sigma_{\text{noise}} = 0.01$.

\subsection{Learning Procedure}
For the initial phase of training, the loss function in~\eqref{eq:J_prob} is minimized using the AdamW optimizer over 100,000 iterations, with a batch size $b = 32$.  From this checkpoint, two distinct training procedures are resumed for additional 700,000 iterations: \textit{standard} training (solving~\eqref{eq:vanilla_loss}) and proposed \textit{robust} training (solving~\eqref{eq:robust_loss}). For standard training, a batch of $b=32$ tasks is sampled at each iteration. In contrast, for robust training, a batch of $b = 80$ tasks is generated and only the 32 tasks ($\alpha = 0.4$) with the highest loss values  are selected for training. {All the other hyperparameters follow standard machine learning practice; more details are provided in the accompanying GitHub repository.} 

{Notably, robust training incurs a longer run time, approximately 29 hours compared to 17.3 hours for standard training. This increase stems mainly from: (i) additional forward passes to evaluate all tasks' risk in each batch before backpropagation, and (ii) the higher cost of dataset generation, performed on the CPU, which accounts for about 36\% of the time difference (roughly 4.2 out of 11.7 hours).}

We remark that in~\cite{wang2023simpleyet}, training starts directly with the robust procedure. However, we found a two-stage training procedure more efficient. In fact, roughly speaking, only after some iterations the trained meta model can discriminate tasks with high and low loss. To enable a direct comparison with a one-stage approach, we also perform training from scratch using the robust procedure. The one-stage and two-stage approaches achieve comparable predictive performance. However, the two-stage approach is 20\% faster than training from scratch (32 {vs 40 hours}).

All the results presented in the rest of the section are obtained in a zero-shot fashion, without any model re-training.

\subsection{In-Distribution Analysis}
The ID performance of standard and robust training is evaluated on a test batch of $b = 20{,}000$ tasks sampled from the same distribution used during training. Table~\ref{tab:in-distribution} reports the average root mean square error (RMSE) over the full test batch, the average over the top 40\% most challenging tasks (8,000 with highest RMSE), and the median RMSE over all the tasks. For comparison, we also report a baseline assuming perfect knowledge of the WH system architecture (i.e., the order of the linear blocks and the structure of the static nonlinearity). For each realization, a separate WH model is fitted using nonlinear least squares (NLSQ). Due to computational cost, results are shown for a representative subset of 2,000 tasks; extending this to all 20,000 would require approximately 5 days.

\begin{table}[htb!]
    \centering
    \caption{RMSE on in-distribution test tasks.}
    \begin{tabular}{cccc}
        \toprule
        \textbf{Training} & \textbf{Test RMSE} & \textbf{Tail RMSE} & \textbf{Median RMSE} \\
        \midrule
        Standard  & 0.088 & 0.173 & 0.050 \\
        Robust    & 0.070 & 0.129 & 0.045 \\
        NLSQ      & 0.289 & 0.700    & 0.023    \\
        \bottomrule
    \end{tabular}
    \label{tab:in-distribution}
\end{table}

Robust training improves tail performance, sthe direct training objective, by over 25\%, while also reducing average RMSE across all tasks by more than 20\%. The higher RMSE of NLSQ is due to convergence issues on critical instances; however, its superior performance on the median confirms that it can outperform meta-models on well-conditioned tasks. The presence of critical outliers in NLSQ predictions is also evident in Fig.~\ref{fig:NSLQ_rmse}. Fig.~\ref{fig:WN_histogram} shows the RMSE distribution, with a zoomed view of the tails highlighting the improved robustness of the proposed approach.

\begin{figure}[htb!]
\centering
\includegraphics[width=.47\textwidth]{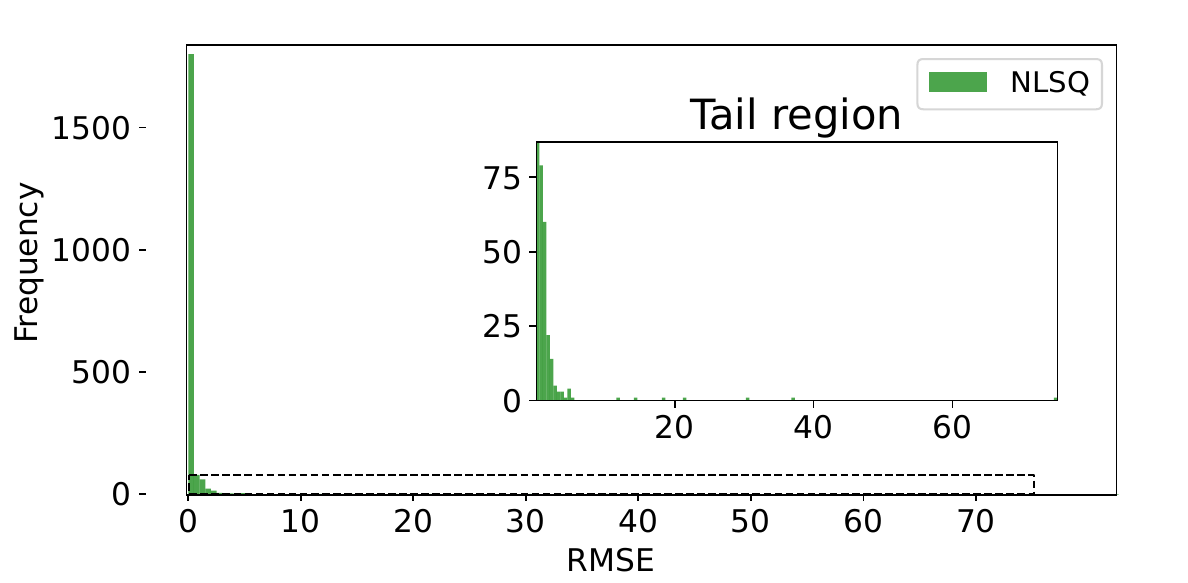}
\caption{In-distribution analysis: RMSE histogram over the test batch obtained with a classic System Identification approach, the non-linear least squares (NLSQ).
}
\label{fig:NSLQ_rmse}
\end{figure}

\begin{figure}[htb!]
\centering
\includegraphics[width=.47\textwidth]{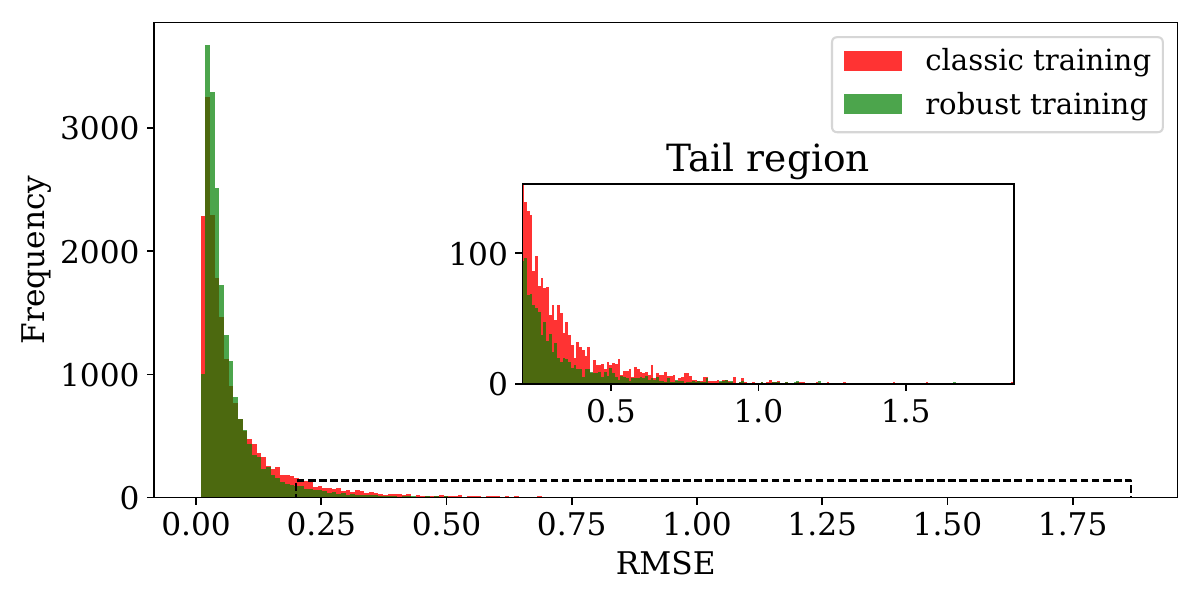}
\caption{In-distribution analysis: RMSE histogram over the test batch obtained with standard and robust training.
}
\label{fig:WN_histogram}
\end{figure}
\subsection{Out-of-Distribution Analysis}
To assess the generalization capability of the robust procedure, we evaluate its performance in two {OOD} scenarios: ($i$) same WH system class but with different input signals and ($ii$) a real-world benchmark from \href{www.nonlinearbenchmark.org}{www.nonlinearbenchmark.org}, where input signal and system dynamics differ from those used in training.\\

\noindent \textit{1) Wiener-Hammerstein System with Different Input Classes}  
We evaluate the performance of the proposed approach on a WH system with two different input signals: ($i$) a Random Binary Signal (RBS) and ($ii$) a random-phase multisine.
For both scenarios, a test batch of $b = 20000$ tasks is used. Table~\ref{tab:out-of-distribution-merged} presents the average and tail RMSE ($8000$ tasks) for both cases. The results indicate a consistent improvement in tail performance (by approximately $18\%$ for multisine and $10\%$ for RBS) and in average performance (by $14\%$ for multisine and $4\%$ for RBS). Fig.~\ref{fig:out-of-distribution} presents the RMSE histograms for the RBS scenario, highlighting the tail performance. This confirms that robust training enhances performance in both average and tail metrics in {OOD} scenarios, showing its generalization capability. \\

\begin{table}[htb!]
    \centering
    \caption{RMSE for out-of-distribution test tasks.}
    \begin{tabular}{cccc}
        \toprule
        \textbf{Input} & \textbf{Training} & \textbf{Test RMSE} & \textbf{Tail RMSE} \\
        \midrule
        RBS      & standard  & 0.232 & 0.604 \\
                  & robust    & 0.223 & 0.547 \\
        \midrule
        multisine & standard  & 0.151 & 0.292 \\
                  & robust    & 0.130 & 0.240 \\
        \bottomrule
    \end{tabular}
    \label{tab:out-of-distribution-merged}
\end{table}

\begin{figure}[htb!]
\centering
        \includegraphics[width=.47\textwidth]{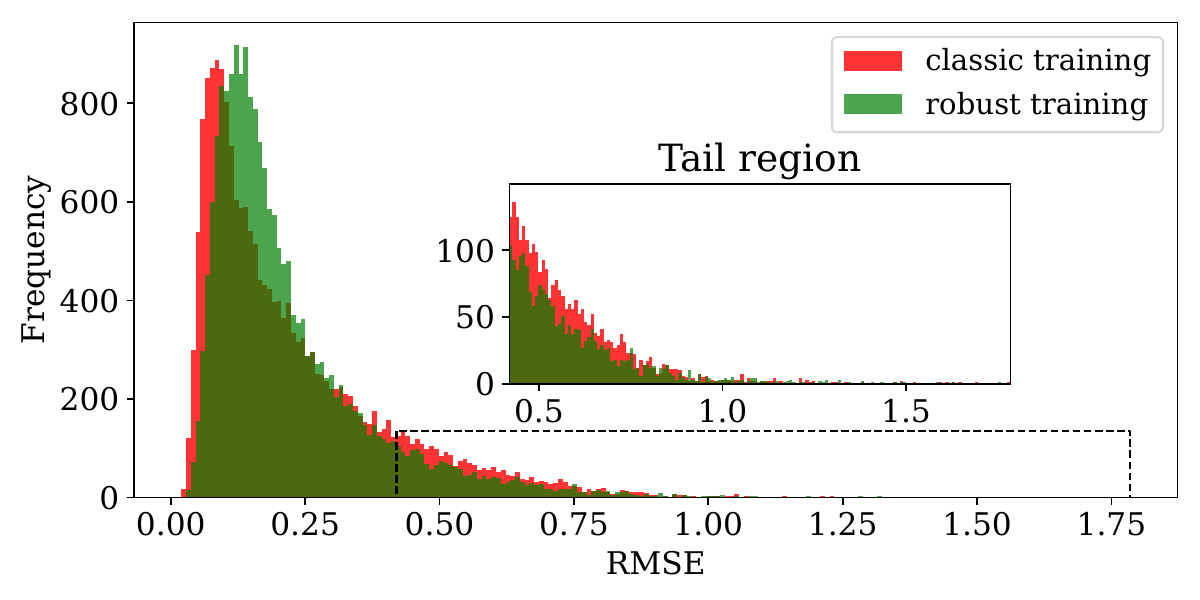}
        \label{fig:PRBS_comparison}  
    \caption{Out-distribution analysis (RBS): RMSE histogram over the test batch obtained with standard and robust training.}
    \label{fig:out-of-distribution}
\end{figure}

\noindent \textit{2) Silverbox Benchmark}

To further evaluate generalization, we test the trained meta models on the Silverbox SYSID benchmark~\cite{wigren2013SB}. The benchmark includes three test sets with real data, generated using different input signals:
\begin{itemize}
    \item a random-phase multisine;  
    \item an ``arrow'' signal with linearly increasing amplitude. The arrow signal is used to generate two test inputs:  
    \begin{itemize}
        \item \textit{arrow (no extrapolation)}: the initial portion of the arrow, where the amplitude remains within the amplitude range of the training input;  
        \item \textit{full arrow}: the complete arrow signal, whose amplitude exceeds that of the training in its last portion.  
    \end{itemize}
\end{itemize}

Since the meta model provides predictions only up to 100 time steps ahead,  
an \textit{iterative inference} approach is adopted to simulate the whole test datasets of the benchmark. The first inference step utilizes the $\nin$ initial conditions provided by the benchmark. In subsequent iterations, the final $\nin$ simulated samples from the previous step are used as initial conditions for the next iteration. This procedure is repeated for all non-overlapping chunks of length 100 within the test dataset. 

Additionally, this iterative inference approach is independently applied to all non-overlapping subsequences of length 400 of the Silverbox training dataset, which are processed by the meta model as context. The final RMSE values, averaged over all iterative inference iterations and training dataset subsequences, are reported in Table~\ref{tab:silverbox}.
\begin{table}[htb!]
    \centering
    \caption{RMSE for Silverbox benchmark test datasets.}
    \begin{tabular}{ccc}
        \toprule
        \textbf{Input} & \textbf{Training} & \textbf{Test RMSE {[V]}}  \\
        \midrule
        multisine      & standard  & 0.00928  \\
                      & robust    & 0.00837 \\
        \midrule
        arrow (no extrapolation) & standard  & 0.00685 \\
        & {robust} &    {0.00631} \\
        \midrule
        full arrow  & standard  & 0.01088 \\
                    & robust    & 0.01079 \\
        \bottomrule
    \end{tabular}
    \label{tab:silverbox}
\end{table}
We observe that robust training improves performance for \textit{multisine} (by approximately 10\%) and \textit{arrow (no extrapolation)} (by approximately 8\%) test inputs.
For \textit{full arrow}, performance is nearly identical between the standard and robust procedures. These results confirm that the robust training approach enhances generalization of the meta model.

\subsection{In-Depth Evaluation under equal Training Time and Dataset Budget}

To ensure a fair and comprehensive evaluation of the proposed robust approach, two additional comparison scenarios are considered. The results in the previous subsections focused on comparing the \emph{standard} and \emph{robust} training procedures, both after the same number of iterations (700{,}000). However, two other factors can be employed for comparison: (i) equal training time and (ii) equal number of datasets generated during training.

The second factor is particularly relevant because, in the \emph{robust} approach, a batch of 80 different systems is generated at each iteration, but only 32 (i.e., 40\%) are used during backpropagation. In contrast, the \emph{standard} approach generates a batch of 32 systems per iteration, all of which are used. As a result, although both approaches undergo the same number of training iterations in the initial comparison, the total number of generated datasets differs significantly.

To evaluate the implications of this difference, the trained meta models are tested across all previously introduced scenarios. Table~\ref{tab:data-analysis} reports the resulting test RMSE values, both over the entire batch (20,000 systems) and the upper tail (8,000 systems), for WH systems excited by white-noise, multisine, and RBS input signals, as well as for the Silverbox benchmark. The performance of the \emph{robust} model is included as a reference. For clarity, the best-performing RMSE values are shown in \textbf{bold}, and the second-best are \underline{underlined}.

\subsubsection{Training Time}

In this scenario, the \emph{standard} training from Section~\ref{subsec:learning_procedure} is extended to match the training time of the \emph{robust} approach. This results in a total of approximately 1.2 million training iterations. The corresponding performance is reported in the second row of Table~\ref{tab:data-analysis}.

\begin{table*}[htb!]
\small
    \centering
    \caption{RMSE over the test batch for different scenarios, and for different methods.}
    \begin{tabular}{ccccccc}
        \toprule \textbf{Method}&
        \textbf{Epochs} & \textbf{Training-time} & \textbf{Case} & \textbf{Test RMSE} & \textbf{Tail RMSE} \\
        \midrule
        Standard ICL&1.75M&40h&white-noise        & \textbf{0.065} & \textbf{0.121} \\
                  &&&multisine    & \textbf{0.119} & \textbf{0.232} \\
                  &&& SilverBox   & [\textbf{0.00820},\textbf{0.00630},0.01100] & \\
                  &&& RBS   & 0.365 & 0.854 \\
        \midrule
        Standard ICL&1.22M&29h&white-noise        & 0.075 & 0.141 \\
                  &&&multisine    & 0.132 & 0.253 \\
                  &&& SilverBox   & [0.00892,0.00665,\underline{0.01094}] & \\
                  &&& RBS   & \underline{0.342} & \underline{0.810} \\
         \midrule
        Robust ICL&{0.7M}&{29h}&white-noise     & \underline{0.070} & \underline{0.129} \\
                  &&&mutlisine     & \underline{0.130} & \underline{0.240} \\
                  &&& SilverBox& {[\underline{0.00837},\underline{0.00631},\textbf{0,01079}]}& \\
                  &&& RBS &\textbf{0.223} & \textbf{0.547} \\
        \bottomrule
    \end{tabular}
    \label{tab:data-analysis}
\end{table*}

From the results in Table~\ref{tab:data-analysis}, we observe that although the \emph{standard} ICL model is trained on approximately $74\%$ more data than in the previously analyzed scenario, the \emph{robust} approach consistently outperforms it across all benchmarks. The performance improvements range from marginal to significant, depending on the specific input scenario. In particular, when comparing this setting to the standard model trained with 700{,}000 iterations, we note that the increased data volume enhances performance in the in-distribution case and in certain out-of-distribution ones. However, in the most challenging OOD setting, characterized by a RBS input signal, the standard ICL model exhibits a marked performance drop, with an RMSE increase of nearly $50\%$ (approximately $47\%$).

We attribute this degradation to \emph{meta-overfitting}~\cite{hospedales2021meta,kirsch2022general}, wherein the meta model begins to memorize specific training instances rather than capturing the generalizable structure of the task distribution. This mirrors classical overfitting, but at the meta-learning level. Evidence for this phenomenon includes not only the degradation under RBS excitation but also the modest gains in the multisine and Silverbox benchmarks, whose input characteristics closely resemble the white-noise, dominated training distribution, supporting the hypothesis that the model overfits to this dominant input type.

\subsubsection{Dataset Budget}

In this scenario, the \emph{standard} training is continued for 1.75 million iterations, corresponding to the same total number of datasets generated in the \emph{robust} approach. This includes all systems generated during training, even those discarded in the robust setting due to low estimated risk. Specifically, since only $40\%$ of the systems are used for backpropagation in the \emph{robust} method, matching the total number of generated systems requires $1{,}750{,}000 = (700{,}000 / 0.4)$ iterations of the \emph{standard} training procedure.

The results for this setting are presented in the first row of Table~\ref{tab:data-analysis}. The observed trends are consistent with the previous comparison: while the \emph{standard} model slightly outperforms the \emph{robust} one in the ID and in some OOD cases, its performance deteriorates significantly in the most challenging OOD scenario (RBS), with an RMSE drop exceeding $50\%$ compared to the result in Table~\ref{tab:out-of-distribution-merged}.

To provide a more comprehensive view, Figure~\ref{fig:distribution_comparison} shows test RMSE histograms for the \emph{standard} ICL (trained for 1,750,000 iterations) and the \emph{robust} ICL (trained for 700,000 iterations) under white-noise, multisine, and RBS inputs. These results confirm the table-based analysis: while the \emph{standard} model performs slightly better under white-noise and multisine inputs, it suffers substantial meta-overfitting under RBS excitation, as reflected in the heavy tail of the RMSE histogram.

\begin{figure*}[htb!]
    \centering
    \begin{subfigure}[b]{0.32\textwidth}
        \centering
        \includegraphics[width=\textwidth]{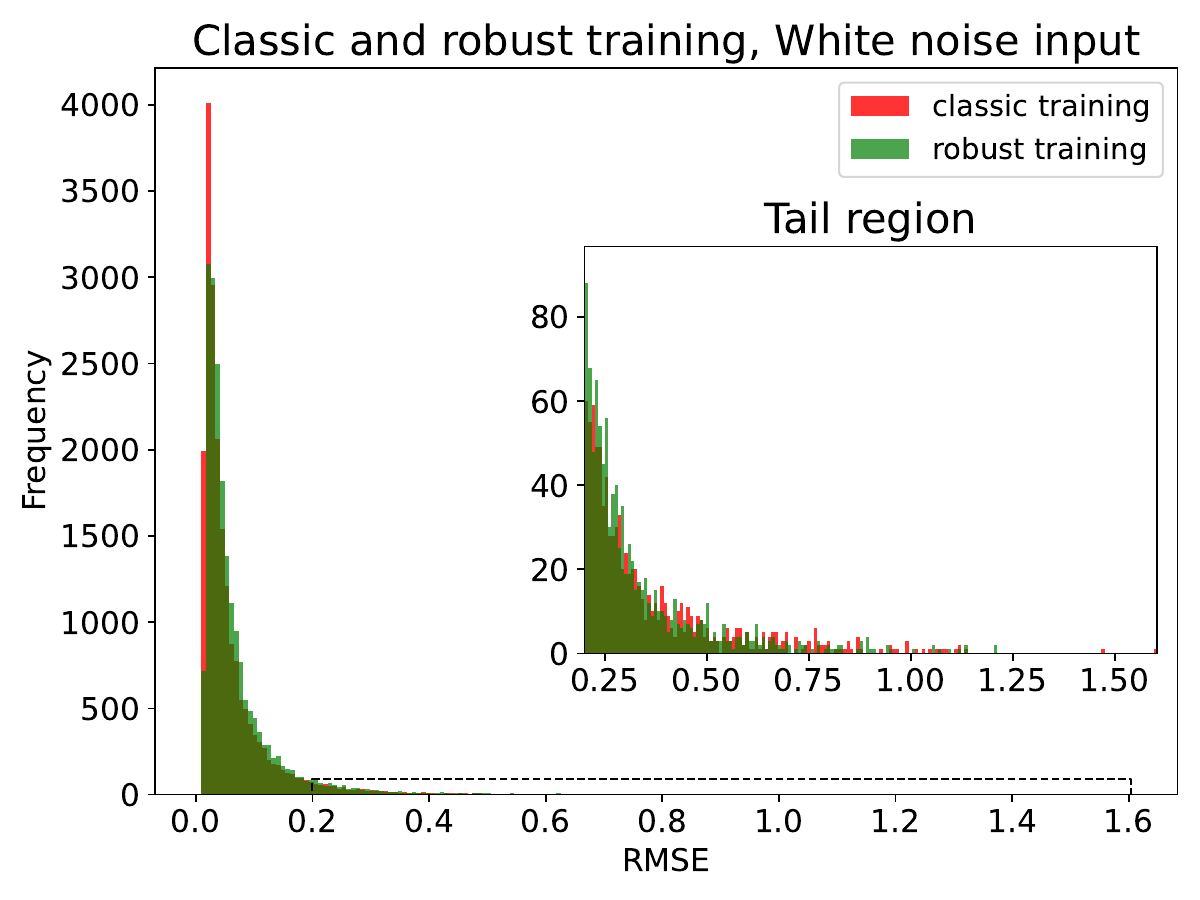}
        \caption{In-distribution with input white-noise}
        \label{fig:in_dist_white_noise}
    \end{subfigure}
    \hfill
    \begin{subfigure}[b]{0.32\textwidth}
        \centering
        \includegraphics[width=\textwidth]{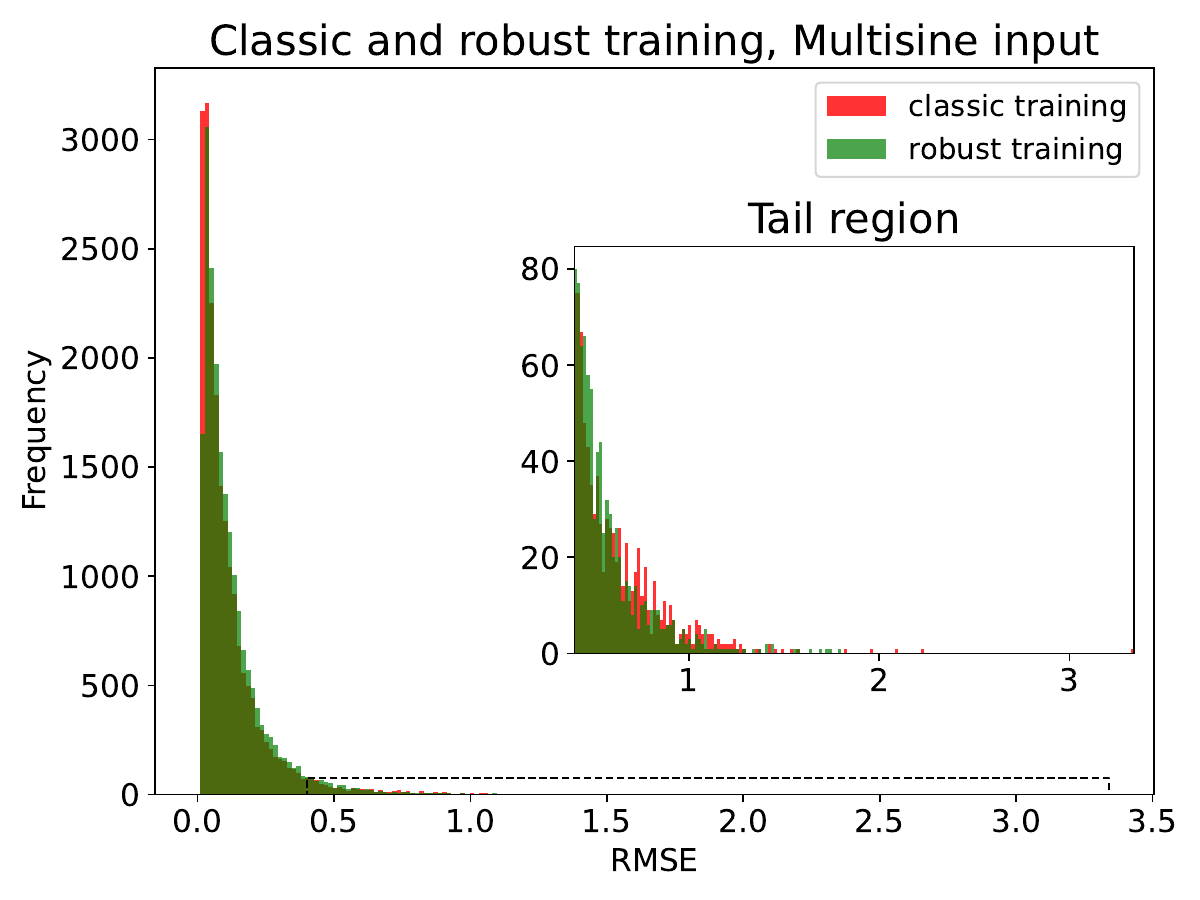}
        \caption{Out-of-distribution with multisine input}
        \label{fig:out_dist_multisine}
    \end{subfigure}
    \hfill
    \begin{subfigure}[b]{0.32\textwidth}
        \centering
        \includegraphics[width=\textwidth]{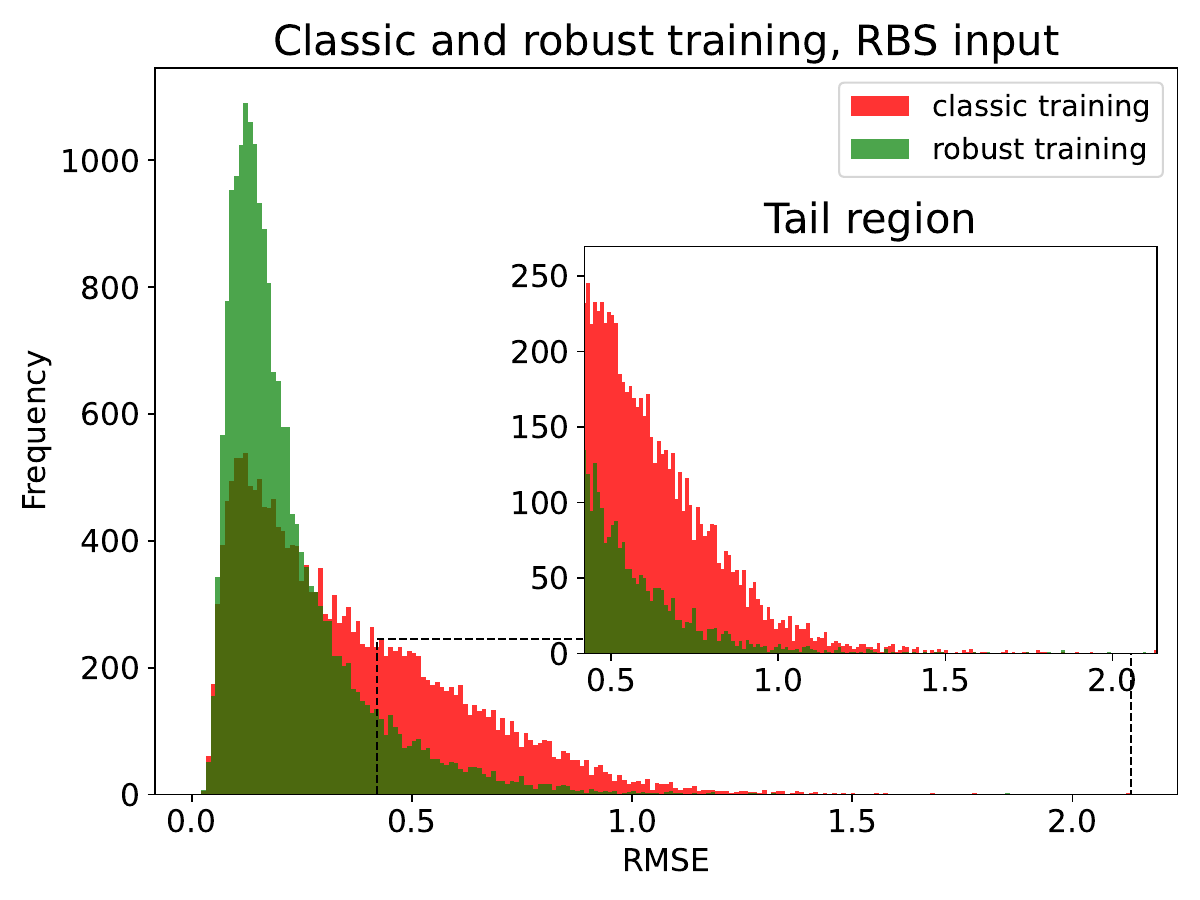}
        \caption{Out-of-distribution with RBS input}
        \label{fig:out_dist_rbs}
    \end{subfigure}
    \caption{Histogram analysis: RMSE over the test batch for different input types (in-distribution and out-of-distribution scenarios) comparing the standard approach trained for 1.75 millions of iteration with the robust one trained for 0.7 millions of iterations.}
    \label{fig:distribution_comparison}
\end{figure*}

\subsection{Role of Uncertainty}

To evaluate the impact of uncertainty estimation in~\eqref{eq:Nodel_free_sim_CI} within the robust learning framework, we conducted an additional experiment. In this setup, the robust training procedure was repeated using the RMSE as the risk function instead of the KL divergence introduced in~\eqref{eq:risk_function}, thereby disregarding the predicted $\sigma$ values. The results are summarized in Table~\ref{tab:data-analysis-rmse}, where the best test RMSE values are highlighted in \textbf{bold}.

\begin{table}[htb!]
\small
    \centering
    \caption{Test RMSE for different risk functions used in the robust ICL framework.}
    \begin{tabular}{lcccc}
        \toprule 
        \textbf{Risk} & \textbf{Case} & \textbf{Test RMSE} & \textbf{Tail RMSE} \\
        \midrule
        RMSE & White-noise        & 0.079 & 0.145 \\
             & Multisine          & 0.143 & 0.268 \\
             & SilverBox          & [9.27, 6.85, 10.88]e-3 & -- \\
             & RBS                & 0.263 & 0.679 \\
        \midrule
        KL   & White-noise        & \textbf{0.070} & \textbf{0.129} \\
             & Multisine          & \textbf{0.130} & \textbf{0.240} \\
             & SilverBox          & \textbf{[8.37, 6.31, 10.79]e-3} & -- \\
             & RBS                & \textbf{0.223} & \textbf{0.547} \\
        \bottomrule
    \end{tabular}
    \label{tab:data-analysis-rmse}
\end{table}

The results demonstrate that incorporating uncertainty estimates through $\sigma$ significantly enhances the robustness of the learning process. In particular, Figure~\ref{fig:rmse_comparison} reports the validation RMSE (computed on held-out tasks sampled every 2{,}000 iterations) over the full batch and over its ``tail'' (i.e., the worst-performing tasks). It is evident that the KL-based robust method consistently achieves lower validation errors compared to the RMSE-based counterpart.

\begin{figure*}[htb!]
    \centering
    \begin{subfigure}[b]{0.45\textwidth}
        \centering
        \includegraphics[width=\textwidth]{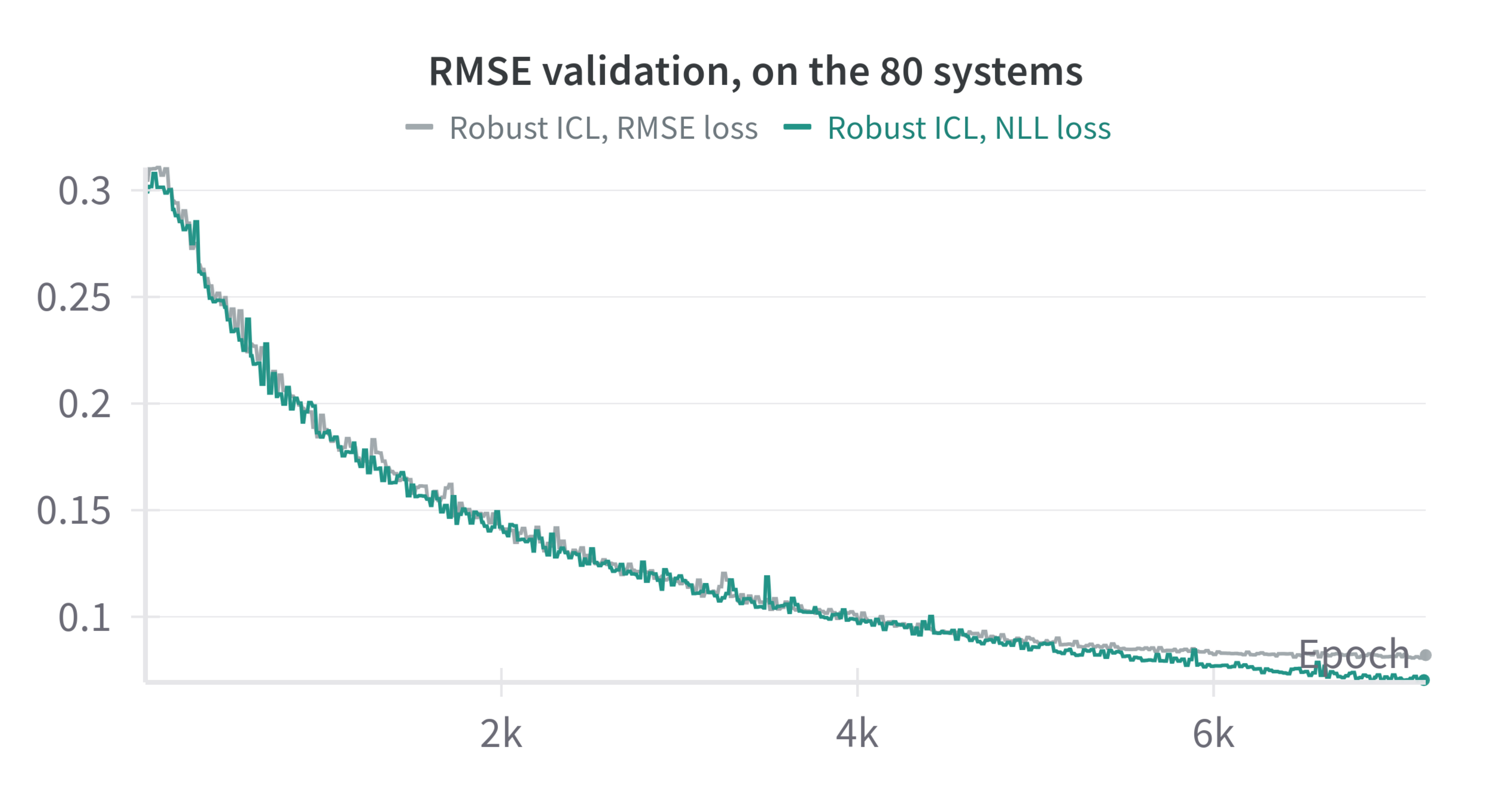}
        \caption{Validation RMSE over the entire batch}
        \label{fig:rmse_val}
    \end{subfigure}
    \hfill
    \begin{subfigure}[b]{0.45\textwidth}
        \centering
        \includegraphics[width=\textwidth]{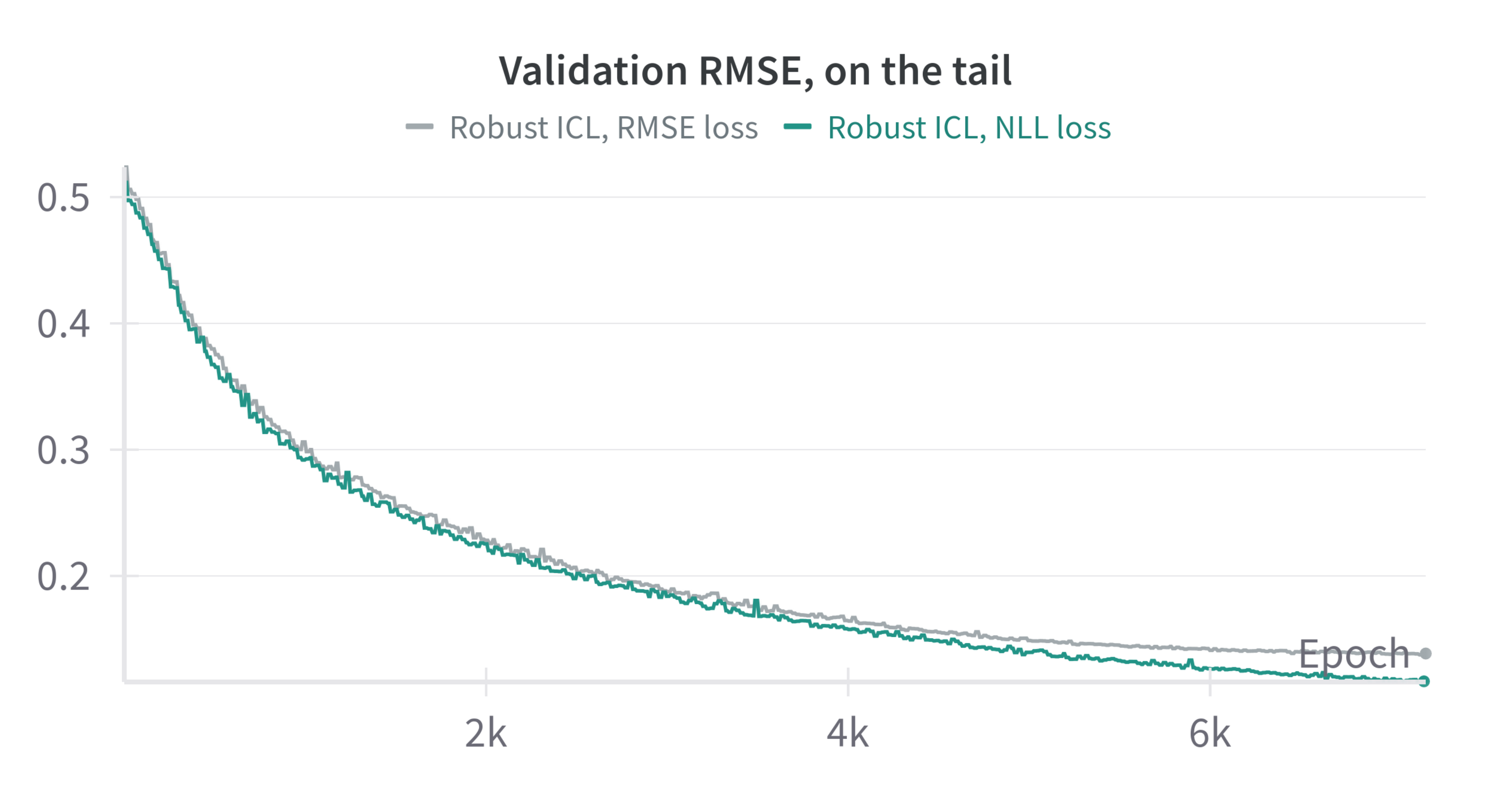}
        \caption{Validation RMSE over the batch tail}
        \label{fig:rmse_val_tail}
    \end{subfigure}
    \caption{Comparison of validation losses using robust ICL with RMSE-based and KL-based risk functions.}
    \label{fig:rmse_comparison}
\end{figure*}

While the Transformer-based meta model functions as a black box, one possible interpretation is that the KL risk does not merely penalize prediction errors, but also encourages the model to calibrate its confidence via the predicted uncertainty $\sigma$. This additional signal appears to guide the learning process toward more robust solutions, especially in heterogeneous and challenging task distributions.

These findings are particularly insightful, as they demonstrate that the effectiveness of the proposed robust approach is not solely due to focusing on high-error tasks. Instead, the choice of the risk function plays a pivotal role. The uncertainty-aware KL divergence introduces a second-order signal, confidence in predictions, that enables the meta model to assess the reliability of its outputs. This in turn facilitates more informed risk evaluation and allows the robust optimization paradigm to outperform standard error-based formulations.

\subsection{Tail tuning}
To complete the analysis, we conducted an tuning study to evaluate the impact of the hyperparameter \(\alpha\) used during robust training. Specifically, we compared the original setting \(\alpha_0 = 0.4\) (selecting the worst 32 out of 80 tasks) against two alternatives: \(\alpha_1 = 0.6\) (selecting the worst 36 out of 60 tasks) and \(\alpha_2 = 0.2\) (selecting the worst 32 out of 160 tasks). The goal of this analysis is twofold: to validate the initial choice of \(\alpha\), and to assess the sensitivity of the approach to this parameter. While a carefully chosen \(\alpha\) can enhance performance, an unsuitable value may significantly degrade accuracy, in which case classical training may be preferable.

Table~\ref{tab:data-analysis-alpha} reports the performance metrics for all three configurations, showing that \(\alpha_0\) consistently outperforms the others across most of the evaluated settings. These findings are further illustrated in Fig.~\ref{fig:tail_comparison}, which shows the RMSE distributions on test data under white-noise, multisine, and RBS inputs. The histograms not only confirm the numerical trends in the table but also reflect the expected behavior: smaller \(\alpha\) values (tighter tails) improve robustness in challenging cases but tend to sacrifice average-case performance.
\begin{table}[htb!]
\small
    \centering
    \caption{Test RMSE for different $\alpha$ used in the robust ICL framework.}
    \begin{tabular}{lcccc}
        \toprule 
        \textbf{Tail} & \textbf{Case} & \textbf{Test RMSE} & \textbf{Tail RMSE} \\
        \midrule
        $\alpha_0$   & White-noise        & \textbf{0.070} & \underline{0.129} \\
             & Multisine          & \textbf{0.130} & \underline{0.240} \\
             & SilverBox          & \textbf{[8.37, 6.31, 10.79]e-3} & -- \\
             & RBS                & \underline{0.223} & \underline{0.547} \\
        \midrule
        $\alpha_1$ & White-noise        & 0.079 & 0.134 \\
             & Multisine          & 0.141 & 0.246 \\
             & SilverBox          & [\underline{9.12}, \underline{6.51}, \underline{11.38}]e-3 & -- \\
             & RBS                & \textbf{0.213} & 0.552 \\
        \midrule
        $\alpha_2$ & White-noise        & \underline{0.076} & \textbf{0.125}\\
             & Multisine          & \underline{0.133} & \textbf{0.238} \\
             & SilverBox          & [10.1, 8.24, 11.64]e-3 & -- \\
             & RBS                & 0.252 & \textbf{0.492} \\
        \bottomrule
    \end{tabular}
    \label{tab:data-analysis-alpha}
\end{table}
This trade-off is also evident in the Silverbox benchmark. While \(\alpha_2 = 0.2\) yields the lowest overall performance (see Table~\ref{tab:data-analysis-alpha},  where the best-performing RMSE values are shown in \textbf{bold}, and the second-best are \underline{underlined}.), it performs best in the extrapolation portion of the task i.e., the region that lies outside the training amplitude. As shown in Fig.~\ref{silv_full_arrow}, although \(\alpha_2\) performs worse in the initial part of the trajectory, it yields superior predictions in the extrapolation region. This is also reflected in the RMSE scores computed exclusively over the extrapolation segment, where the values for \(\alpha_0\), \(\alpha_1\), and \(\alpha_2\) are 0.028, 0.030, and 0.024, respectively. This highlights the potential of the robust approach to generalize better in the most difficult, out-of-distribution regions, depending on the choice of \(\alpha\).
\begin{figure*}[htb!]
    \centering
    \begin{subfigure}[b]{0.32\textwidth}
        \centering
        \includegraphics[width=\textwidth]{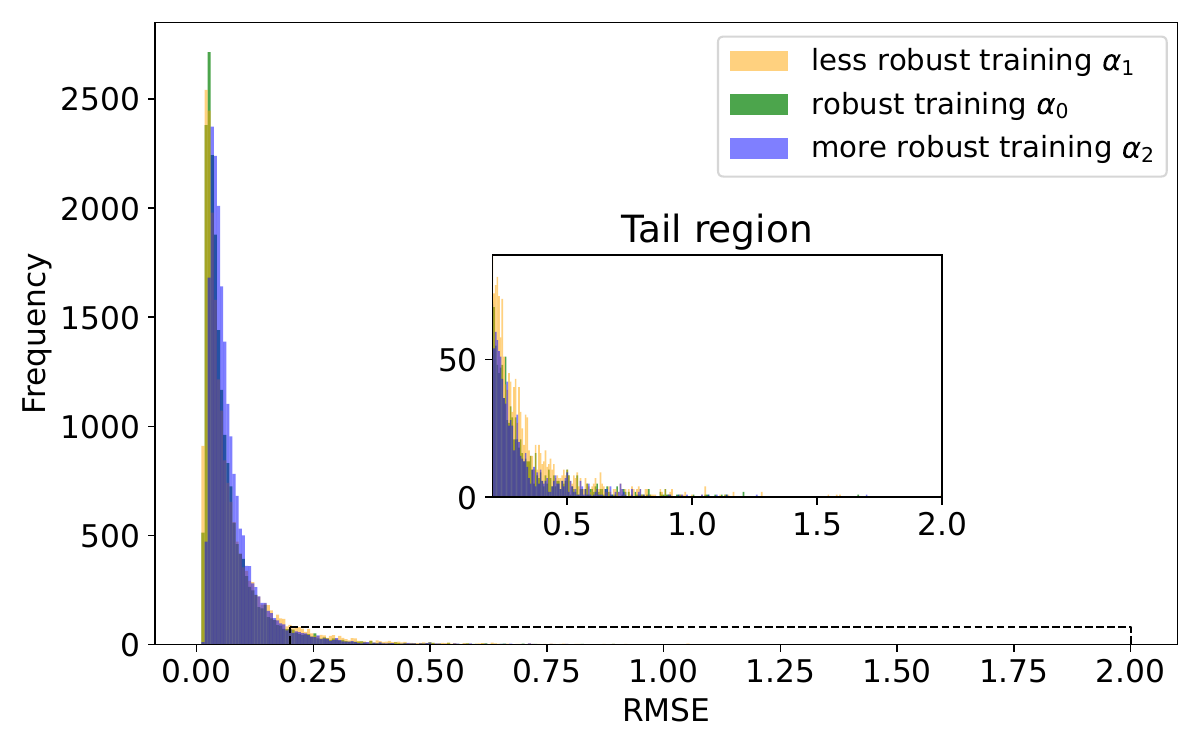}
        \caption{In-distribution with input white-noise}
        \label{fig:in_dist_WN_robust}
    \end{subfigure}
    \hfill
    \begin{subfigure}[b]{0.32\textwidth}
        \centering
        \includegraphics[width=\textwidth]{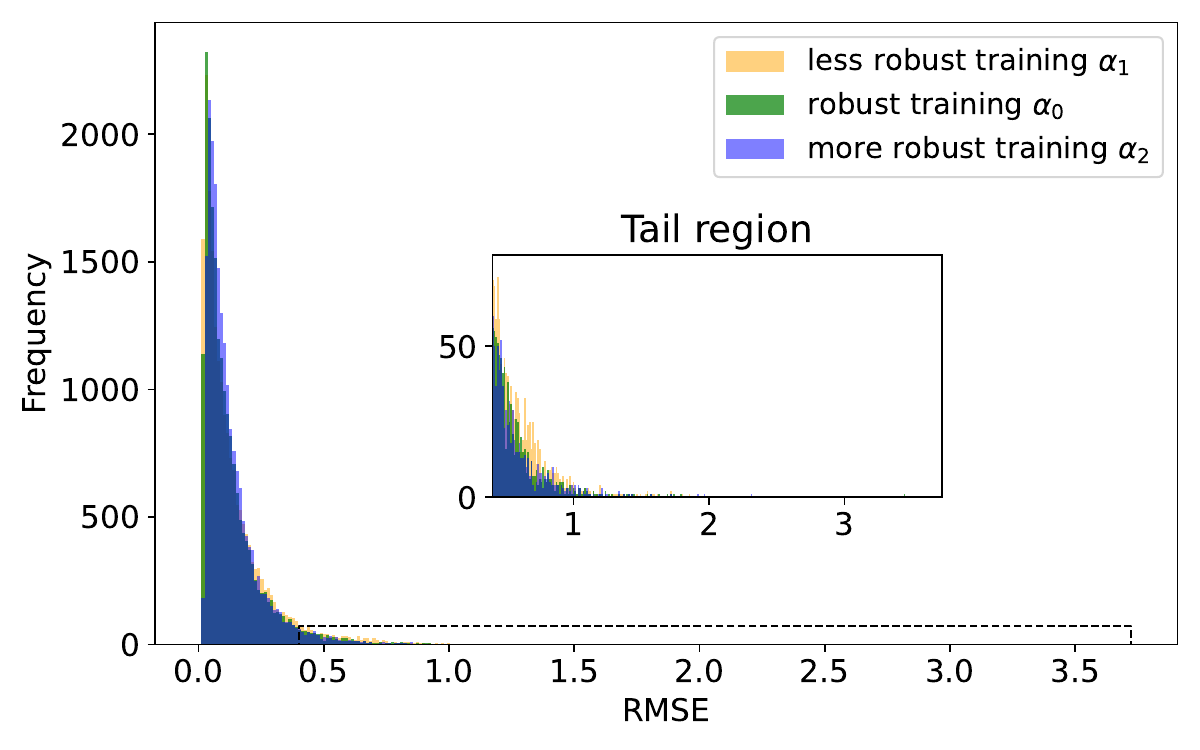}
        \caption{Out-of-distribution with multisine input}
        \label{fig:out_in_dist_multisine_robust}
    \end{subfigure}
    \hfill
    \begin{subfigure}[b]{0.32\textwidth}
        \centering
        \includegraphics[width=\textwidth]{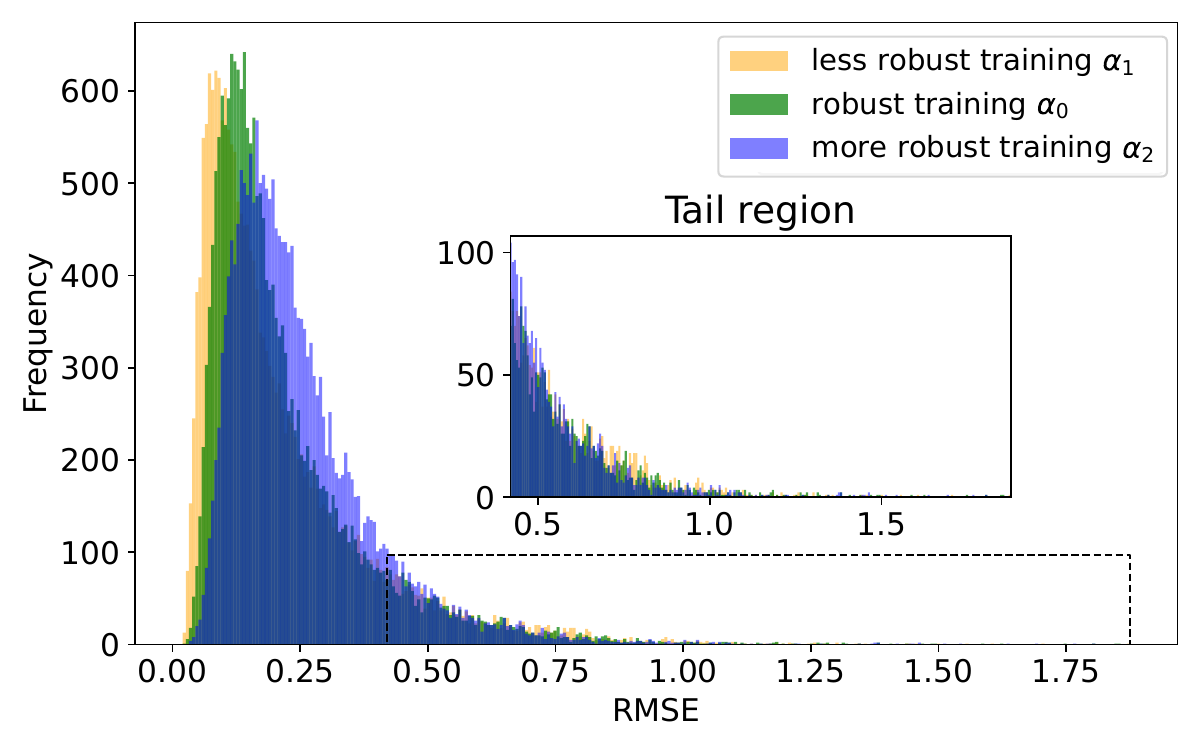}
        \caption{Out-of-distribution with RBS input}
        \label{fig:out_in_dist_rbs_robust}
    \end{subfigure}
    \caption{Histogram analysis: RMSE over the test batch for different input types (in-distribution and out-of-distribution scenarios) for the different choice of the tail treshold $\alpha$.}
    \label{fig:tail_comparison}
\end{figure*}

 \begin{figure*}[!hbt]
\centering
\includegraphics[width=.8\textwidth]{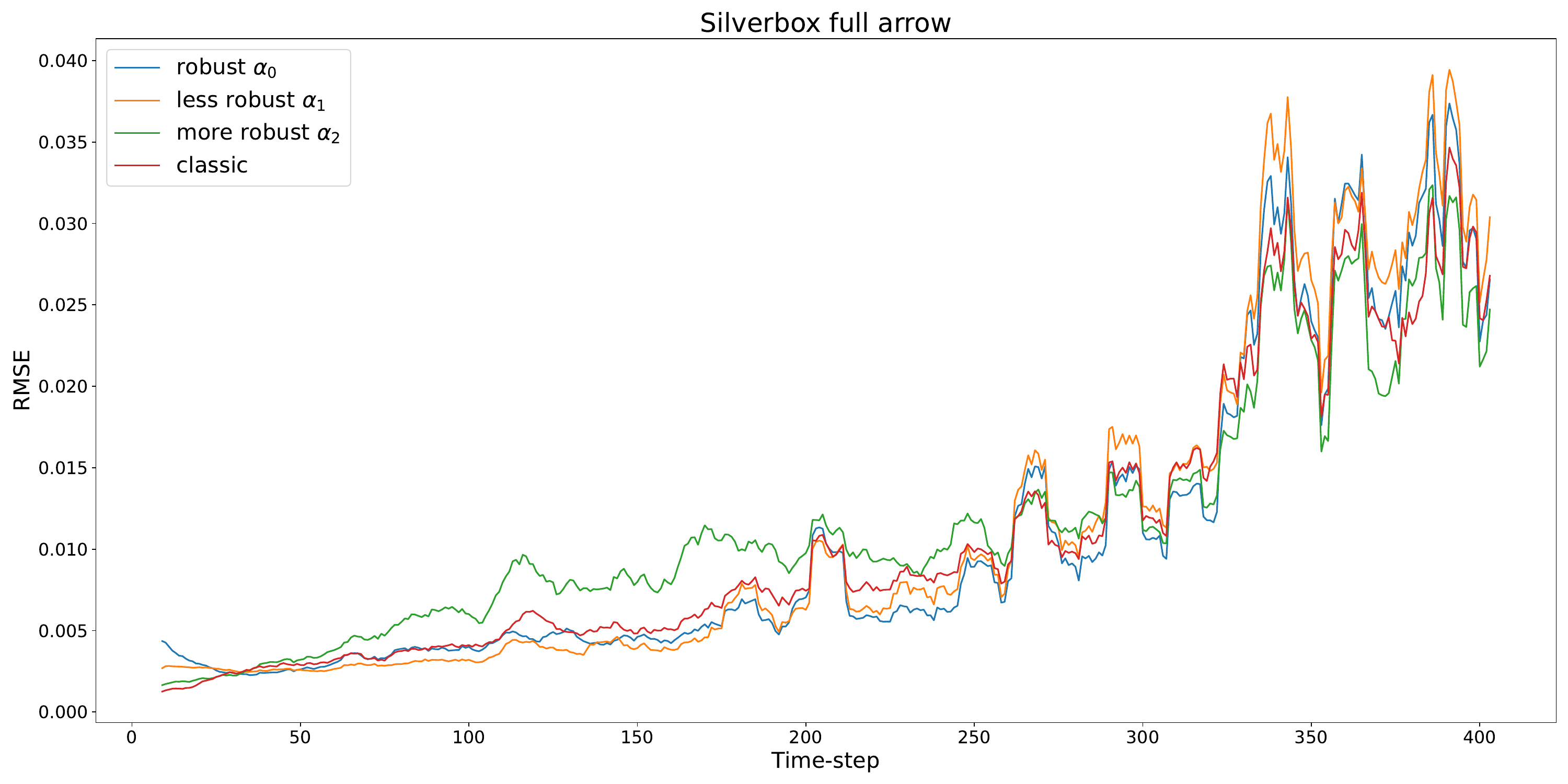}
\caption{Silverbox benchmark \textit{full-arrow} analysis: RMSE plot over the the time-step comparing classic and robust training with different tail tresholds $\alpha$.
}
\label{silv_full_arrow}
\end{figure*}

\section{Conclusions}
\label{sec:conc}

This work addresses meta-learning from a distributional robustness standpoint by reformulating the standard expected loss minimization to focus on high-risk tasks. Specifically, the objective is to minimize the expected tail loss over the task distribution, yielding a more robust meta-learning framework. This approach improves both stability and generalization of the resulting meta model. Numerical results show consistent gains not only in tail performance but also across the full task distribution. Robustness is further confirmed in {OOD} scenarios, where the method achieves reliable improvements.

Current research activities focus on extending the proposed robust learning approach in different directions:
\begin{itemize}
    \item \textbf{Meta-learning with alternative architectures}, such as structured state-space models or diffusion models. This requires appropriately modifying the assumptions in Proposition~\ref{prop:local_opt} to ensure convergence of the gradient-based robust optimization algorithm;
    \item \textbf{Stochastic optimal control}, with an emphasis on worst-case scenarios to enhance the robustness and safety of the designed controller.

    \item \textbf{Data-efficient training}, to reduce reliance on large synthetic datasets, that often requires prior knowledge of the system, by reusing simulated systems across epochs, while limitating overfitting.
    
    \item \textbf{Theoretical refinement}, focused on reformulating the assumptions of Lemma~\ref{lemma:local_opt_paper} from a more practical standpoint and deriving a stronger theorem that provides local optimality guarantees under more realistic conditions;
    
    \item \textbf{Uncertainty analysis}, both quantitatively, by increasing the number of training iterations, and qualitatively, through a theoretical investigation aimed at elucidating the role of uncertainty in guiding robust learning.

\end{itemize}

\bibliographystyle{IEEEtran} 
\bibliography{root} 
\end{document}